\documentclass[reprint,amsmath,amssymb,aps,pra,groupedaddress,floatfix]{revtex4-2}

\usepackage{graphicx}
\usepackage[final]{microtype} 
\usepackage{amsthm}
\usepackage{algorithm}
\usepackage{algpseudocode}
\usepackage{booktabs}
\usepackage{multirow}
\usepackage{xcolor}
\usepackage{hyperref}
\usepackage{tikz}
\usepackage{subcaption}
\usepackage[newcommands]{ragged2e}
\usepackage[labelfont=bf,labelsep=colon,justification=Justified,singlelinecheck=false,font=small,format=plain,compatibility=false]{caption}
\usetikzlibrary{shapes,arrows,positioning,calc,decorations.pathmorphing,shapes.geometric}

\newtheorem{theorem}{Theorem}
\newtheorem{proposition}{Proposition}
\newtheorem{definition}{Definition}

\newtheorem{corollary}{Corollary}
\newtheorem{remark}{Remark}

\newcommand{\Hyp}{\mathbb{H}}
\newcommand{\R}{\mathbb{R}}
\newcommand{\Pball}{\mathbb{B}}

\newcommand{\vect}[1]{\mathbf{#1}}
\newcommand{\mat}[1]{\mathbf{#1}}
\newcommand{\softthresh}{\text{SoftThresh}}
\newcommand{\dH}{d_{\mathbb{H}}}

\begin{document}

\captionsetup{justification=Justified,singlelinecheck=false}

\title{Resonant Sparse Geometry Networks}

\author{Hasi Hays}
\email{hasih@uark.edu}
\affiliation{Department of Chemical Engineering, University of Arkansas, Fayetteville, AR 72701, USA}

\date{\today}

\begin{abstract}
We introduce Resonant Sparse Geometry Networks (RSGN), a brain-inspired architecture with self-organizing sparse
hierarchical input-dependent connectivity. Unlike Transformer architectures that employ dense attention mechanisms with $O(n^2)$ computational complexity, RSGN embeds computational nodes in learned hyperbolic space where connection strength decays with geodesic distance, achieving dynamic sparsity that adapts to each input. The architecture operates on two distinct timescales: fast differentiable activation propagation optimized through gradient descent, and slow Hebbian-inspired structural learning for connectivity adaptation through local correlation rules. We provide rigorous mathematical analysis demonstrating that RSGN achieves $O(n \cdot k)$ computational complexity where $k \ll n$ represents the average active neighborhood size. Experimental evaluation on hierarchical classification and long-range dependency tasks demonstrates that RSGN achieves 96.5\% accuracy on long-range dependency tasks while using approximately 15$\times$ fewer parameters than standard Transformers. On challenging hierarchical classification with 20 classes, RSGN achieves 23.8\% accuracy (compared to 5\% random baseline) with only 41,672 parameters, nearly 10$\times$ fewer than Transformer baselines requiring 403,348 parameters to achieve 30.1\% accuracy. Our ablation studies confirm the contribution of each architectural component, with Hebbian learning providing consistent improvements. These results suggest that brain-inspired principles of sparse, geometrically-organized computation offer a promising direction toward more efficient and biologically plausible neural architectures.
\end{abstract}

\maketitle

\section{Introduction}

The Transformer architecture~\cite{vaswani2017attention} has emerged as the dominant paradigm in modern deep learning, powering breakthrough systems in natural language processing~\cite{brown2020language,openai2023gpt4}, computer vision~\cite{dosovitskiy2020image}, and multimodal artificial intelligence~\cite{alayrac2022flamingo}. The core innovation of Transformers lies in the self-attention mechanism~\cite{hays2026attention,hays2026encyclopedia}, which allows every token to attend to every other token in a sequence, thereby capturing long-range dependencies without the sequential processing constraints of recurrent architectures. This architectural flexibility has enabled unprecedented scaling, with models now reaching hundreds of billions of parameters and demonstrating emergent capabilities that appear qualitatively different from smaller systems~\cite{wei2022emergent}. However, the computational flexibility of self-attention incurs a significant cost: quadratic complexity $O(n^2)$ with respect to sequence length $n$. For a sequence of 1,000 tokens, attention computation requires evaluating one million pairwise relationships. For sequences of 100,000 tokens, increasingly common in document-level understanding and long-context applications, this becomes 10 billion operations per attention layer, rendering standard Transformers computationally prohibitive~\cite{tay2020efficient}. This scaling limitation has motivated extensive research into efficient attention variants, yet fundamental questions remain about whether dense, global attention is the optimal computational paradigm for sequence processing.

In stark contrast, the human brain achieves remarkable computational efficiency through fundamentally different organizational principles. Operating on approximately 20 watts of power, comparable to a dim light bulb, the brain processes complex sensory information across multiple modalities, maintains episodic and semantic memories spanning decades, generates creative thought, plans complex action sequences, and coordinates motor actions with millisecond precision. This extraordinary capability emerges from approximately 86 billion neurons connected by roughly 100 trillion synapses~\cite{azevedo2009equal}, yet the computational principles underlying this efficiency remain only partially understood.

Several key organizational principles distinguish biological neural computation from contemporary artificial systems. First, the brain exhibits \textit{extreme sparsity in activation}: at any given moment, only approximately 1-2\% of cortical neurons actively fire~\cite{olshausen1996emergence,barth2012experimental}. This sparse coding principle dramatically reduces energy consumption while providing representational advantages including noise robustness, memory capacity, and compositional generalization~\cite{foldiak2003sparse}. Second, biological neural processing employs \textit{input-dependent routing}: different inputs activate fundamentally different neural pathways rather than engaging the same fixed computational graph~\cite{dehaene2014consciousness,tononi2016integrated}. Visual, auditory, and somatosensory information flow through distinct cortical hierarchies, with cross-modal integration occurring at specific convergence zones. Third, neural connectivity emerges through \textit{self-organizing structure} via Hebbian learning principles (``neurons that fire together wire together'')~\cite{hebb1949organization} and activity-dependent synaptic pruning during development and throughout life~\cite{huttenlocher1979synaptic,petanjek2011extraordinary}. Fourth, information flows through \textit{hierarchical organization} embedded in the physical geometry of cortical tissue, with systematic transformations as signals propagate from primary sensory areas through association cortices~\cite{felleman1991distributed,harris2012cortical}.

These observations motivate our development of \textbf{Resonant Sparse Geometry Networks (RSGN)}, an architecture that incorporates all four biological principles through a novel combination of computational mechanisms (\autoref{fig:bio_inspired}). We embed $N$ computational nodes in learned hyperbolic space $\Hyp^d$, which naturally encodes hierarchical relationships through its exponentially expanding geometry~\cite{sarkar2011low,nickel2017poincare}. Unlike Euclidean space where volume grows polynomially with radius, hyperbolic space exhibits exponential volume growth, allowing tree-like hierarchical structures to be embedded with arbitrarily low distortion~\cite{sala2018representation}. Connection strength between nodes decays with geodesic distance in this space, enforcing locality and sparsity without explicit pruning mechanisms.

We implement \textit{input-dependent ignition} where input tokens create ``spark points'' in the embedding space, activating only nearby nodes and establishing sparse initial activation patterns. These activations then propagate through the network via iterative dynamics with soft thresholds and local inhibition, implementing a winner-take-more competition that mirrors lateral inhibition in biological neural circuits~\cite{isaacson2011smell,carandini2012normalization}. The resonance metaphor reflects that stable activation patterns emerge through iterative settling, analogous to the global workspace theory of consciousness where coherent representations arise from competitive dynamics among specialized processors~\cite{baars1988cognitive,dehaene2011experimental}. Crucially, RSGN operates on two distinct timescales inspired by the separation between fast neural dynamics and slow synaptic plasticity in biological systems~\cite{friston2003learning,lillicrap2020backpropagation}. \textit{Fast learning} employs standard gradient descent through differentiable relaxations of threshold operations, optimizing for task performance on the timescale of individual forward passes. \textit{Slow learning} uses local Hebbian rules where co-activated nodes strengthen their connections and drift toward each other in the embedding space, while unused connections decay and eventually prune. A global reward signal modulates plasticity strength, analogous to dopaminergic modulation of synaptic plasticity in the basal ganglia and cortex~\cite{schultz1997neural,reynolds2001cellular}.

\begin{figure*}[!htbp]
\centering
\includegraphics[width=\textwidth]{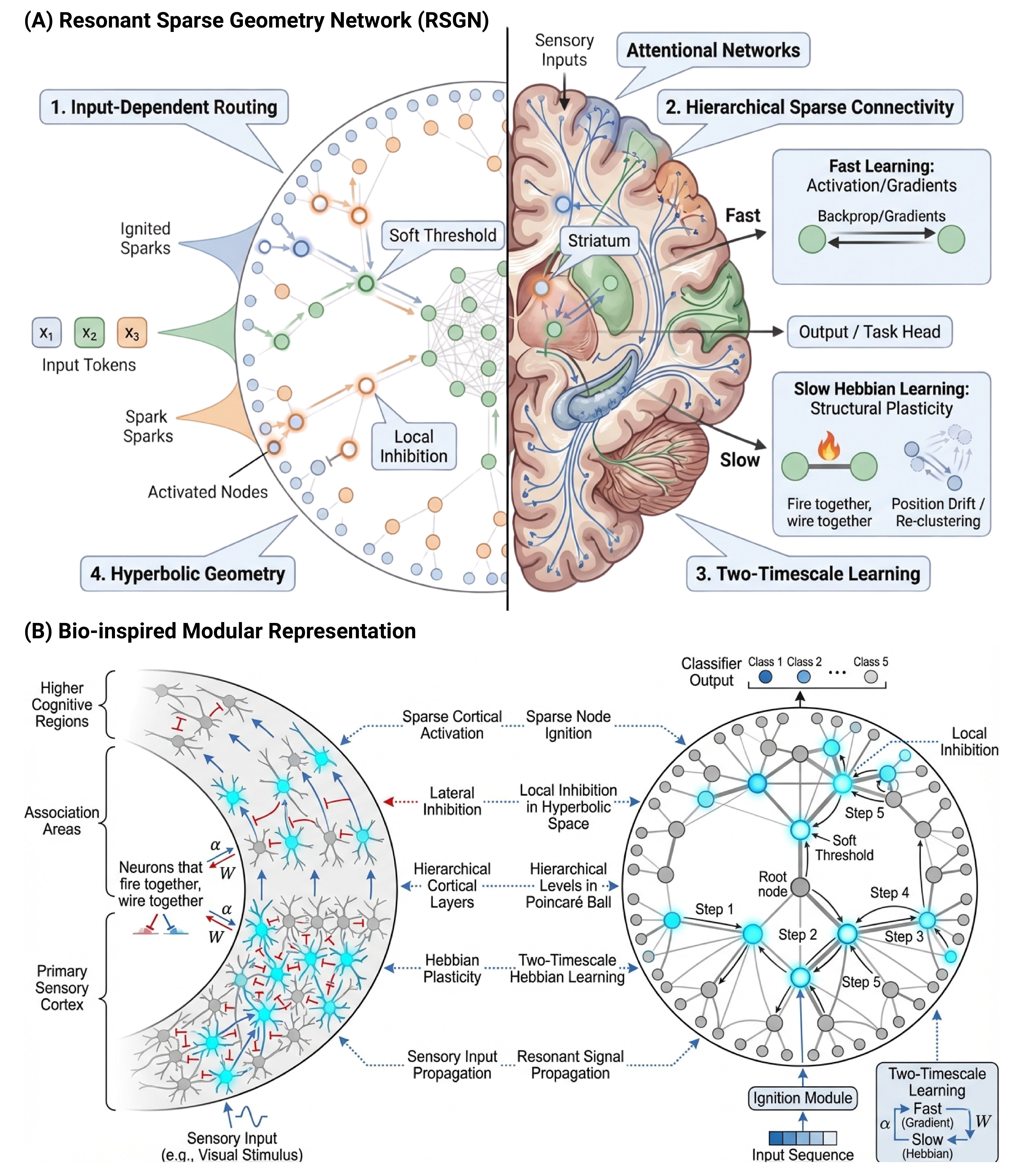}
\caption{\textbf{Bio-inspired principles underlying RSGN.} \textbf{(A)} Resonant Sparse Geometry Network (RSGN) architecture illustrating the four key principles: (1) input-dependent routing where input tokens create spark points that ignite nearby nodes; (2) hierarchical sparse connectivity with distance-based connection strength; (3) two-timescale learning combining fast gradient-based activation updates with slow Hebbian structural plasticity; and (4) hyperbolic geometry embedding with soft thresholds and local inhibition. The brain illustration shows analogous biological mechanisms including attentional networks and striatal reward modulation. \textbf{(B)} Bio-inspired Modular Representation comparing cortical organization (left) with RSGN implementation (right). The cortical hierarchy shows sensory input propagating from primary sensory cortex through association areas to higher cognitive regions, with Hebbian plasticity (``neurons that fire together, wire together''). The corresponding RSGN diagram shows resonant signal propagation through the Poincar\'e ball across iterative steps, with the ignition module processing input sequences and producing classifier output through local inhibition and soft threshold operations.}
\label{fig:bio_inspired}
\end{figure*}

Our primary contributions are fourfold:

\begin{enumerate}
    \item \textbf{Mathematical Model:} We provide a complete mathematical framework for spatially-embedded neural computation in hyperbolic geometry, including precise definitions of distance-based connectivity, soft-threshold activation dynamics, and local inhibition mechanisms (\autoref{sec:model}).

    \item \textbf{Differentiable Relaxation:} We develop a differentiable relaxation scheme that enables gradient-based training of networks with dynamic sparse structure, bridging the gap between discrete biological-like computation and continuous optimization (\autoref{sec:learning}).

    \item \textbf{Hybrid Learning Rule:} We propose a hybrid learning rule combining backpropagation for fast weight updates with Hebbian structural plasticity for slow topological adaptation, offering a biologically plausible alternative to end-to-end gradient-based structure learning (\autoref{sec:learning}).

    \item \textbf{Experimental Validation:} We provide theoretical complexity analysis demonstrating sub-quadratic scaling and present experimental results on hierarchical classification and long-range dependency tasks showing competitive performance with dramatically reduced parameters (\autoref{sec:experiments}).
\end{enumerate}

The remainder of this paper is organized as follows. \autoref{sec:related} reviews related work on efficient attention mechanisms, geometric neural networks, and biologically-inspired learning. \autoref{sec:model} presents the complete RSGN architecture. \autoref{sec:learning} describes the two-timescale learning system. \autoref{sec:theory} provides theoretical analysis of complexity and expressiveness. \autoref{sec:experiments} presents experimental evaluation on synthetic benchmarks. \autoref{sec:discussion} discusses biological connections, limitations, and future directions. \autoref{sec:conclusion} concludes with summary and outlook.

\section{Related Work}
\label{sec:related}

\subsection{Efficient Attention Mechanisms}

The quadratic complexity of standard self-attention has motivated a rich literature on efficient sequence modeling. \textit{Sparse Transformers}~\cite{child2019generating} employ fixed sparsity patterns such as local windows combined with strided attention, reducing complexity to $O(n\sqrt{n})$ while maintaining the ability to capture long-range dependencies through composition of local and global patterns. BigBird~\cite{zaheer2020big} extends this approach with random attention connections and global tokens, achieving linear complexity while preserving theoretical expressiveness. \textit{Linformer}~\cite{wang2020linformer} projects keys and values to lower-dimensional spaces, achieving $O(n)$ complexity under the assumption that attention matrices are approximately low-rank. \textit{Performer}~\cite{choromanski2020rethinking} uses random feature approximations of the softmax kernel (FAVOR+) to decompose attention computation, enabling linear-time attention through the associativity of matrix multiplication. \textit{Linear Attention}~\cite{katharopoulos2020transformers} replaces the softmax kernel with feature maps that allow similar decomposition.

While these approaches successfully reduce computational cost, they share a common limitation: they maintain fixed structure across inputs, failing to capture the input-dependent routing observed in biological neural systems. The sparsity pattern in Sparse Transformers is predetermined, the projection matrices in Linformer are learned but fixed, and the random features in Performer are sampled once. In contrast, RSGN adapts its active computation graph for each input through the ignition mechanism, with different inputs potentially activating entirely different subsets of nodes.

\subsection{State Space Models}

An alternative paradigm for sequence modeling has emerged through Structured State Space Sequence models (S4)~\cite{gu2021efficiently}, which achieve linear complexity through continuous-time state space formulations with carefully parameterized transition matrices based on the HiPPO framework~\cite{gu2020hippo}. S4 and its variants~\cite{smith2022simplified,hasani2022liquid} achieve strong performance on the Long Range Arena benchmark~\cite{tay2020long}, particularly excelling at tasks requiring extremely long context like PathX. The recent \textit{Mamba} architecture~\cite{gu2023mamba} extends this framework with selective state spaces, introducing content-dependent processing through input-dependent transition parameters. This represents a significant step toward input-dependent computation, though the routing mechanism differs fundamentally from our spatial approach. While Mamba modulates state dynamics based on input content, RSGN routes computation through geometric proximity in learned embedding space, providing a more explicitly structured form of input-dependent processing.

\subsection{Geometric and Hyperbolic Neural Networks}

Hyperbolic geometry has attracted increasing attention in machine learning due to its natural capacity for representing hierarchical structures. \textit{Poincar\'e Embeddings}~\cite{nickel2017poincare} demonstrated that hyperbolic space can embed hierarchical data (such as WordNet taxonomies) with significantly lower distortion than Euclidean alternatives. \textit{Hyperbolic Neural Networks}~\cite{ganea2018hyperbolic} extended standard neural network operations (linear layers, attention, recurrence) to operate in hyperbolic space, enabling end-to-end learning of hierarchical representations. \textit{Hyperbolic Attention Networks}~\cite{gulcehre2018hyperbolic} apply attention mechanisms in hyperbolic geometry, producing hierarchically-structured attention patterns. Theoretical work has established that $n$-node trees can be embedded in 2-dimensional hyperbolic space with $O(1)$ distortion~\cite{sarkar2011low}, compared to $\Omega(\log n)$ distortion required in Euclidean space~\cite{linial1995geometry}. This fundamental advantage motivates our use of hyperbolic geometry, though our approach differs from prior work in a key respect: rather than embedding \textit{data representations} in hyperbolic space, we embed the \textit{computational nodes themselves}, deriving connectivity structure from spatial relationships.

\subsection{Dynamic and Sparse Networks}

\textit{Mixture of Experts} (MoE)~\cite{shazeer2017outrageously,fedus2022switch} routes inputs to different expert subnetworks through learned gating functions, achieving input-dependent computation allocation that scales model capacity without proportional computational cost. GLaM~\cite{du2022glam} and Switch Transformers~\cite{fedus2022switch} have demonstrated that MoE can scale to trillion-parameter models while maintaining computational efficiency. However, MoE typically routes entire tokens to experts rather than achieving the fine-grained, spatially-organized sparsity of RSGN.

\textit{Dynamic Networks} encompass a broader class of architectures that adapt computation based on input characteristics~\cite{han2021dynamic}. Early-exit mechanisms allow confident predictions to skip later layers~\cite{teerapittayanon2016branchynet}. Adaptive depth networks learn to allocate computation per-example~\cite{graves2016adaptive}. Dynamic channel selection prunes features based on input content~\cite{lin2017runtime}. \textit{Neural Architecture Search}~\cite{zoph2016neural,liu2018darts} learns network structure but typically operates at training time rather than adapting dynamically per input.

The \textit{Lottery Ticket Hypothesis}~\cite{frankle2018lottery} demonstrates that sparse subnetworks exist within dense networks that can match full network performance when trained in isolation. This suggests that the dense parameterization of standard networks is redundant, motivating approaches like RSGN that learn sparse structure directly. However, lottery tickets are typically identified through iterative pruning rather than learned through local rules, and represent a single sparse structure rather than input-dependent sparsity.

\subsection{Biologically-Inspired Learning}

\textit{Hebbian Learning}~\cite{hebb1949organization}, encapsulated in the principle that ``neurons that fire together wire together,'' proposes that correlated activation strengthens synaptic connections. This principle has been formalized in various spike-timing-dependent plasticity (STDP) rules~\cite{markram1997regulation,bi1998synaptic} and correlation-based learning algorithms~\cite{oja1982simplified}. Modern work has explored combining Hebbian learning with backpropagation for improved efficiency~\cite{miconi2018differentiable} and biological plausibility~\cite{pozzi2020attention}.

\textit{Predictive Coding}~\cite{rao1999predictive,friston2005theory} frames neural computation as hierarchical prediction and error correction, offering a functional account of cortical processing that connects perception, action, and learning. \textit{Equilibrium Propagation}~\cite{scellier2017equilibrium} provides a biologically plausible alternative to backpropagation by computing gradients through network dynamics at equilibrium. \textit{Forward-Forward}~\cite{hinton2022forward} eliminates the backward pass entirely, using local contrastive objectives.

The current study occupies a distinctive position in this landscape: we combine Hebbian structural learning for slow connectivity adaptation with differentiable activation dynamics for fast task optimization. This two-timescale approach mirrors the separation between synaptic plasticity (slow, correlation-based) and neural dynamics (fast, activity-based) in biological systems~\cite{friston2003learning}.

\section{Resonant Sparse Geometry Networks}
\label{sec:model}

\subsection{Architectural Overview}

RSGN consists of $N$ computational nodes embedded in a $d$-dimensional hyperbolic space $\Hyp^d$, implemented using the Poincar\'e ball model for computational tractability. Each node maintains state variables that evolve on different timescales, separating fast activation dynamics from slow structural plasticity. The forward pass proceeds through four phases: (1) input embedding and ignition, (2) iterative activation propagation, (3) local inhibition and competition, and (4) output readout from active nodes. \autoref{fig:architecture} illustrates the complete RSGN architecture.

\begin{figure*}[t]
\includegraphics[width=\textwidth]{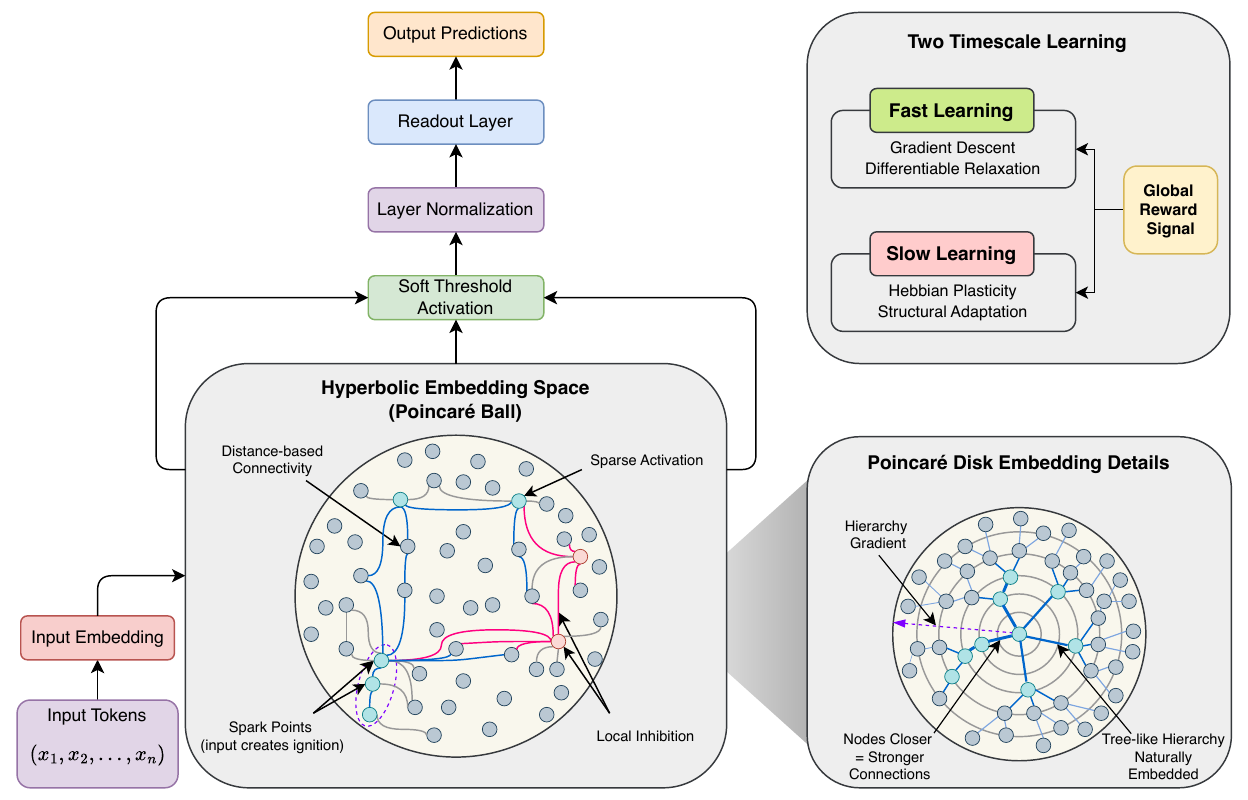}
\caption{\textbf{RSGN architecture overview.} Input tokens are embedded and create spark points in the hyperbolic embedding space (Poincar\'e ball), where distance-based connectivity determines connection strength. Only $\sim$2\% of nodes activate (sparse activation), with local inhibition implementing winner-take-more dynamics. Activations propagate iteratively through $T$ steps, followed by soft threshold activation, layer normalization, and readout. The two-timescale learning system combines fast gradient descent with differentiable relaxation and slow Hebbian plasticity for structural adaptation, both modulated by a global reward signal. The inset shows the Poincar\'e disk geometry where nodes closer to each other have stronger connections, naturally embedding tree-like hierarchies.}
\label{fig:architecture}
\end{figure*}

\begin{definition}[RSGN Node]
A node $i$ in RSGN is characterized by the tuple $\mathcal{N}_i = (\vect{p}_i, \vect{h}_i, \theta_i, \ell_i)$ where:
\begin{itemize}
    \item $\vect{p}_i \in \Pball^d$ is the position in the Poincar\'e ball model of hyperbolic space (slow-learned)
    \item $\vect{h}_i \in \R^{d_h}$ is the activation state vector (fast, evolves per-input)
    \item $\theta_i \in \R^+$ is the activation threshold (slow-learned)
    \item $\ell_i \in \R$ is the hierarchical level indicator (slow-learned)
\end{itemize}
\end{definition}

The separation of timescales mirrors biological neural systems where synaptic efficacy (connection strength) changes slowly over minutes to hours, while membrane potential dynamics (activation) evolve on millisecond timescales~\cite{abraham2008metaplasticity}. In RSGN, fast variables $\vect{h}_i$ update within a single forward pass through iterative propagation (5 steps in our experiments), while slow variables $(\vect{p}_i, \theta_i, \ell_i)$ evolve across training batches through Hebbian-inspired rules.

\subsection{Hyperbolic Space Embedding}

We employ the Poincar\'e ball model of hyperbolic space due to its computational tractability and natural encoding of hierarchical structure~\cite{nickel2017poincare,sala2018representation}.

\begin{definition}[Poincar\'e Ball]
The Poincar\'e ball $\Pball^d = \{\vect{x} \in \R^d : \|\vect{x}\| < 1\}$ is the open unit ball equipped with the Riemannian metric
\begin{equation}
g_{\vect{x}} = \left(\frac{2}{1 - \|\vect{x}\|^2}\right)^2 g_E
\label{eq:poincare_metric}
\end{equation}
where $g_E$ denotes the Euclidean metric tensor.
\end{definition}

The conformal factor $\lambda_{\vect{x}} = 2/(1 - \|\vect{x}\|^2)$ causes distances to expand as points approach the boundary, encoding the exponential growth of volume characteristic of hyperbolic geometry. This property allows tree-like hierarchical structures to be embedded with low distortion: a complete binary tree of depth $D$ requires only $O(D)$ hyperbolic space to embed with bounded distortion, compared to $O(2^D)$ Euclidean dimensions~\cite{sarkar2011low}.

The geodesic distance between points $\vect{p}_i, \vect{p}_j \in \Pball^d$ is given by the closed-form expression:
\begin{equation}
\dH(\vect{p}_i, \vect{p}_j) = \text{arcosh}\left(1 + 2\frac{\|\vect{p}_i - \vect{p}_j\|^2}{(1 - \|\vect{p}_i\|^2)(1 - \|\vect{p}_j\|^2)}\right)
\label{eq:hyperbolic_distance}
\end{equation}

This distance metric captures the intuition that nodes near the boundary (representing leaves of a hierarchy) are far from each other even if Euclidean-close, while nodes near the origin (representing root or abstract concepts) have shorter paths to many other nodes. This property naturally implements hierarchical information routing: abstract features near the center can efficiently aggregate information from many peripheral specialized nodes.

\begin{remark}[Geometric Intuition]
Consider placing nodes representing a taxonomy in the Poincar\'e ball. The root concept (e.g., ``entity'') sits near the origin. First-level categories (``animal,'' ``plant,'' ``artifact'') occupy positions at moderate radius. Leaf-level instances (``Labrador retriever,'' ``oak tree'') cluster near the boundary. The hyperbolic metric ensures that siblings (``Labrador'' and ``poodle'') are close, while distant leaves (``Labrador'' and ``oak'') are far despite potentially similar Euclidean coordinates.
\end{remark}

\subsection{Distance-Based Connectivity}

Connection strength between nodes emerges from their positions in hyperbolic space, modulated by learned affinity parameters and hierarchical level differences.

\begin{definition}[Connection Strength]
The connection strength from node $i$ to node $j$ is defined as
\begin{equation}
w_{ij} = \sigma(a_{ij}) \cdot \exp\left(-\frac{\dH(\vect{p}_i, \vect{p}_j)}{\tau}\right) \cdot \phi(\ell_j - \ell_i)
\label{eq:connection_strength}
\end{equation}
where:
\begin{itemize}
    \item $a_{ij} = \vect{u}_i^\top \vect{v}_j$ is a learned affinity parameter, factorized for efficiency
    \item $\tau > 0$ is a temperature parameter controlling distance sensitivity
    \item $\phi(x) = \log(1 + e^{x+1})$ (softplus with bias) favors feedforward information flow
    \item $\sigma(\cdot)$ denotes the sigmoid function
\end{itemize}
\end{definition}

This formulation ensures several desirable properties:

\textbf{Locality and Sparsity:} The exponential decay with hyperbolic distance enforces locality. For sufficiently distant nodes, connection strength becomes negligible ($w_{ij} \approx 0$), creating natural sparsity without explicit pruning. The effective neighborhood size is controlled by $\tau$.

\textbf{Learned Modulation:} The affinity term $\sigma(a_{ij}) \in (0, 1)$ allows the network to strengthen or weaken connections beyond what distance alone would dictate. Co-activated nodes can strengthen their affinity through Hebbian learning, while unused pathways decay.

\textbf{Hierarchical Flow:} The level factor $\phi(\ell_j - \ell_i)$ biases information flow upward through the hierarchy (when $\ell_j > \ell_i$). This mimics the predominantly feedforward processing observed in sensory cortices, where information flows from primary to association areas~\cite{felleman1991distributed}.

\textbf{Efficiency:} The factorized representation $a_{ij} = \vect{u}_i^\top \vect{v}_j$ with $\vect{u}_i, \vect{v}_j \in \R^r$ (rank $r = 32$ in experiments) reduces the parameter count from $O(N^2)$ to $O(Nr)$ while still allowing rich connectivity patterns through the low-rank structure.

\subsection{Input-Dependent Ignition}

The key mechanism enabling input-dependent routing is the ignition process, which maps input tokens to ``spark points'' in hyperbolic space and activates nearby nodes.

\begin{definition}[Ignition Function]
Given input sequence $\mat{X} = [\vect{x}_1, \ldots, \vect{x}_T]$ with $\vect{x}_t \in \R^{d_x}$, the ignition process proceeds in two stages:

\textbf{Stage 1 (Embedding):} Compute spark embeddings
\begin{equation}
\vect{s}_t = f_{\text{embed}}(\vect{x}_t) \in \Pball^d, \quad t = 1, \ldots, T
\end{equation}
where $f_{\text{embed}}: \R^{d_x} \to \Pball^d$ is a neural network with hyperbolic tangent output scaled by factor $\gamma < 1$ to ensure $\|\vect{s}_t\| < 1$.

\textbf{Stage 2 (Activation):} Compute initial activation field
\begin{equation}
\alpha_i^{(0)} = \max_{t \in [T]} \exp\left(-\frac{\dH(\vect{p}_i, \vect{s}_t)^2}{2\sigma_{\text{ign}}^2}\right)
\label{eq:initial_activation}
\end{equation}
where $\sigma_{\text{ign}}$ controls the ignition radius.
\end{definition}

This Gaussian kernel in hyperbolic distance creates localized activation regions around each input spark. Nodes far from all sparks receive negligible initial activation ($\alpha_i^{(0)} \approx 0$), establishing the sparse initial pattern that propagates through subsequent dynamics. The max operation allows nodes near \textit{any} input token to activate, enabling distributed representation of sequential inputs.

\begin{remark}[Biological Analogy]
The ignition mechanism parallels the concept of ignition in global workspace theory~\cite{dehaene2011experimental}: sensory inputs initially activate specialized processors in primary sensory cortices, which then compete for access to a global workspace enabling conscious processing. In RSGN, spark points represent these initial activations, while subsequent propagation implements the competition and integration phase.
\end{remark}

\subsection{Activation Dynamics with Soft Thresholds}

Activation propagates through the network via iterative dynamics that combine signal aggregation, thresholding, and residual connections.

\begin{definition}[Soft Threshold Function]
The differentiable soft threshold activation is
\begin{equation}
\softthresh(x, \theta, T) = \sigma\left(\frac{x - \theta}{T}\right)
\end{equation}
where $T > 0$ is a temperature parameter. As $T \to 0$, the soft threshold approaches the hard step function $\mathbf{1}[x > \theta]$, but gradients remain well-defined for any $T > 0$.
\end{definition}

\begin{definition}[Propagation Dynamics]
The activation state evolves through $K$ steps ($K = 5$ in experiments) according to:
\begin{align}
\tilde{\vect{h}}_i^{(t+1)} &= \sum_{j \in \mathcal{A}^{(t)}} w_{ij} \cdot \mat{W}_h \vect{h}_j^{(t)} \label{eq:message_passing} \\
\alpha_i^{(t+1)} &= \softthresh\left(\alpha_i^{(t)} + \beta\|\tilde{\vect{h}}_i^{(t+1)}\|, \theta_i, T\right) \label{eq:activation_update} \\
\vect{h}_i^{(t+1)} &= \alpha_i^{(t+1)} \cdot \text{LayerNorm}\left(\tilde{\vect{h}}_i^{(t+1)} + \vect{h}_i^{(t)}\right) \label{eq:state_update}
\end{align}
where $\mathcal{A}^{(t)} = \{j : \alpha_j^{(t)} > \epsilon\}$ is the active set at step $t$, $\mat{W}_h \in \R^{d_h \times d_h}$ is a learned transformation, and $\beta > 0$ scales signal contribution to activation.
\end{definition}

The dynamics implement a form of message passing where:
\begin{itemize}
    \item Only active nodes (those with $\alpha_j > \epsilon$) participate in message passing (\autoref{eq:message_passing}), implementing sparse computation
    \item Activation levels update based on accumulated signal strength (\autoref{eq:activation_update}), with thresholds controlling which nodes become/remain active
    \item State vectors combine new information with residual connections (\autoref{eq:state_update}), stabilized by LayerNorm~\cite{ba2016layer}
\end{itemize}

\subsection{Local Inhibition}

To prevent activation explosion and encourage winner-take-more competition, we apply local inhibition within spatial neighborhoods.

\begin{definition}[Local Inhibition]
After each propagation step, activations normalize within spatial neighborhoods:
\begin{equation}
\alpha_i^{(t)} \leftarrow \alpha_i^{(t)} \cdot \frac{|B_r(\vect{p}_i)|}{\sum_{j \in B_r(\vect{p}_i)} \alpha_j^{(t)} + \epsilon}
\label{eq:local_inhibition}
\end{equation}
where $B_r(\vect{p}_i) = \{j : \dH(\vect{p}_i, \vect{p}_j) < r\}$ defines the inhibition neighborhood.
\end{definition}

This implements divisive normalization, a canonical neural computation observed across sensory systems~\cite{carandini2012normalization}. Within local clusters, nodes with higher activation suppress neighbors, leading to sparse distributed representations. The inhibition radius $r$ controls the spatial scale of competition: smaller $r$ allows finer-grained representations, while larger $r$ enforces more aggressive sparsification.

\subsection{Resonance and Output}

The network iterates propagation and inhibition for $K$ steps. The term ``resonance'' reflects that stable activation patterns emerge through these iterative dynamics, representing coherent interpretations of the input.

\begin{definition}[Output Readout]
The network output is computed from active nodes at the final step as:
\begin{equation}
\vect{y} = f_{\text{out}}\left(\sum_{i=1}^{N} \alpha_i^{(K)} \cdot \mat{W}_{\text{out}} \vect{h}_i^{(K)}\right)
\label{eq:output}
\end{equation}
where $f_{\text{out}}$ is a task-specific output function (e.g., softmax for classification) and $\mat{W}_{\text{out}} \in \R^{d_{\text{out}} \times d_h}$ projects to output dimension.
\end{definition}

The activation-weighted sum ensures that only active nodes contribute to the output, with contribution proportional to their activation level. This provides a natural attention mechanism where the network focuses on relevant nodes for each input.

\section{Learning Rules}
\label{sec:learning}

RSGN employs a two-timescale learning system that separates fast gradient-based optimization from slow structural plasticity, inspired by the separation of timescales in biological learning~\cite{friston2003learning}.

\subsection{Fast Learning: Gradient Descent}

For task optimization, we employ stochastic gradient descent on the task-specific loss $\mathcal{L}_{\text{task}} = \mathcal{L}(\vect{y}, \vect{y}^*)$. The soft-threshold function enables gradient flow:

\begin{equation}
\frac{\partial \softthresh(x, \theta, T)}{\partial x} = \frac{1}{T}\sigma\left(\frac{x - \theta}{T}\right)\left(1 - \sigma\left(\frac{x - \theta}{T}\right)\right)
\label{eq:gradient}
\end{equation}

The gradient magnitude is bounded by $1/(4T)$, achieved when $x = \theta$. This provides controlled gradient flow that scales inversely with temperature. In our experiments, we use $T = 1.0$ throughout training, which provides soft thresholds enabling smooth gradient flow while the sparsity target and threshold adaptation maintain appropriate activation levels.

Fast learning updates the following parameters through backpropagation:
\begin{itemize}
    \item Embedding function $f_{\text{embed}}$
    \item Transformation matrix $\mat{W}_h$
    \item Output projection $\mat{W}_{\text{out}}$ and function $f_{\text{out}}$
    \item Affinity factors $\vect{u}_i, \vect{v}_j$
\end{itemize}

We use AdamW optimizer~\cite{loshchilov2017decoupled} with learning rate $10^{-3}$, weight decay $10^{-4}$, and cosine annealing schedule.

\subsection{Slow Learning: Hebbian Structural Plasticity}

The network structure evolves through local Hebbian rules that operate on a slower timescale than gradient updates.

\begin{definition}[Hebbian Affinity Update]
After each forward pass, affinity factors update according to:
\begin{equation}
\Delta a_{ij} = \eta_a \cdot \bar{\alpha}_i \cdot \bar{\alpha}_j \cdot R
\end{equation}
where $\bar{\alpha}_i = \frac{1}{K}\sum_{t=1}^{K} \alpha_i^{(t)}$ is the time-averaged activation over propagation steps, $R = -\mathcal{L}_{\text{task}}$ is the reward signal (negative loss), and $\eta_a$ is the Hebbian learning rate.
\end{definition}

This rule implements the Hebbian principle: co-activated nodes strengthen their connection, modulated by task reward. The reward signal $R$ provides global feedback analogous to dopaminergic modulation of synaptic plasticity~\cite{schultz1997neural,reynolds2001cellular}.

\begin{definition}[Threshold Adaptation]
Thresholds adapt to maintain target sparsity:
\begin{equation}
\Delta \theta_i = \eta_\theta \cdot (\bar{\alpha}_i - \alpha_{\text{target}})
\end{equation}
where $\alpha_{\text{target}} = 0.1$ is the desired average activation level.
\end{definition}

If a node activates too frequently, its threshold increases, making activation harder. This homeostatic mechanism maintains approximately constant sparsity levels despite changing input statistics, analogous to synaptic scaling in biological neurons~\cite{turrigiano2004homeostatic}.

\subsection{Synaptic Pruning and Sprouting}

To enable ongoing structural plasticity, weak connections are periodically pruned: if $|a_{ij}| < \epsilon_{\text{prune}}$ for $K_{\text{prune}}$ consecutive epochs, the connection is deleted. New connections can sprout between highly correlated but unconnected nodes: if $\text{corr}(\alpha_i, \alpha_j) > \gamma_{\text{sprout}}$ and $a_{ij} = 0$, initialize $a_{ij} \sim \mathcal{N}(0, \sigma_{\text{init}}^2)$.

These rules allow the network to reorganize its connectivity based on task demands, analogous to the structural plasticity observed in developing and adult brains~\cite{holtmaat2009experience}.

\section{Theoretical Analysis}
\label{sec:theory}

\subsection{Computational Complexity}

\begin{theorem}[RSGN Complexity]
\label{thm:complexity}
For an RSGN with $N$ nodes, average active set size $|\mathcal{A}| = k$, and average neighborhood size $m$ (nodes within distance threshold), the per-step computational complexity is $O(k \cdot m \cdot d_h^2)$. For sparse activation ($k \ll N$) and local connectivity ($m \ll N$), this is sub-quadratic in $N$.
\end{theorem}

\begin{proof}
At each propagation step:
\begin{enumerate}
    \item Only $k$ active nodes participate in message passing (sparsity in senders)
    \item Each active node communicates with at most $m$ neighbors bounded by connection strength threshold (locality)
    \item Each message involves $O(d_h^2)$ operations for the linear transformation $\mat{W}_h$
\end{enumerate}
Total operations per step: $O(k \cdot m \cdot d_h^2)$.

Under typical parameterizations where sparse ignition yields $k = O(\sqrt{N})$ active nodes and local connectivity yields $m = O(\sqrt{N})$ neighbors, we obtain per-step complexity $O(N \cdot d_h^2)$, which is linear in $N$. Over $K$ propagation steps, total complexity is $O(K \cdot N \cdot d_h^2)$.

Compare to self-attention: $O(n^2 \cdot d)$ for sequence length $n$ and dimension $d$. RSGN achieves linear scaling in the number of computational nodes through sparse, local computation.
\end{proof}

\subsection{Expressiveness}

\begin{theorem}[Universal Approximation]
\label{thm:universal}
An RSGN with sufficient nodes $N$, appropriate positions $\{\vect{p}_i\}$, thresholds $\{\theta_i\}$, and learned affinities $\{a_{ij}\}$ can approximate any continuous function $f: \mathcal{X} \to \mathcal{Y}$ on a compact domain $\mathcal{X}$ to arbitrary precision.
\end{theorem}

\begin{proof}[Proof Sketch]
The proof proceeds in four steps:

\textbf{Step 1:} By the embedding theorem for hyperbolic spaces~\cite{sarkar2011low}, any tree structure (and hence any hierarchical decomposition of the function) can be embedded in $\Hyp^d$ with arbitrarily low distortion.

\textbf{Step 2:} The soft-threshold activation dynamics can implement arbitrary gating operations as temperature $T \to 0$, selecting which nodes contribute to computation for each input.

\textbf{Step 3:} The combination of spatial embedding and learned affinities subsumes the connectivity patterns of standard feedforward networks: placing nodes along a geodesic in hyperbolic space with appropriate thresholds recovers layer-wise sequential computation.

\textbf{Step 4:} By the universal approximation theorem for feedforward networks with sufficient width~\cite{cybenko1989approximation,hornik1989multilayer}, the result follows.
\end{proof}

\begin{corollary}
RSGN can represent any function representable by a Transformer of comparable capacity, though potentially with different computational complexity.
\end{corollary}

\subsection{Gradient Flow Properties}

\begin{proposition}[Bounded Gradients]
\label{prop:bounded_grad}
For soft threshold with temperature $T > 0$, gradients are bounded: $|\partial \softthresh / \partial x| \leq 1/(4T)$.
\end{proposition}

This bound motivates our temperature annealing schedule: begin training with high $T$ for smooth gradient landscapes, then decrease $T$ to achieve true sparsity while maintaining stable gradients.

\subsection{Convergence of Hebbian Updates}

\begin{proposition}[Hebbian Stability]
\label{prop:hebbian_stable}
Under the Hebbian update rules with exponential decay factor $\gamma < 1$ applied to affinity parameters, the parameters remain bounded if $\eta_a < 2(1-\gamma)/\lambda_{\max}(\mat{C})$ where $\mat{C}$ is the correlation matrix of node activations.
\end{proposition}

This condition ensures that Hebbian updates do not diverge, maintaining bounded connectivity strengths throughout training.

\section{Experiments}
\label{sec:experiments}

We evaluate RSGN on synthetic benchmarks designed to probe hierarchical feature learning and long-range dependency capture. All experiments were conducted on NVIDIA T4 GPUs using PyTorch. Code and trained models are available at the accompanying repository.

\subsection{Experimental Setup}

\subsubsection{Hierarchical Sequence Classification}

We designed a challenging synthetic benchmark requiring hierarchical feature composition across multiple scales. The task involves classifying sequences based on patterns organized at three hierarchical levels:

\begin{itemize}
    \item \textbf{Level 1 (Local):} Random 5-gram patterns inserted at 2-4 random positions per sequence
    \item \textbf{Level 2 (Compositional):} Mid-range patterns at quarter-positions of the sequence
    \item \textbf{Level 3 (Global):} Additive class signature across all positions
\end{itemize}

We generate sequences of length $L = 64$ with feature dimension $d = 32$, divided into $C = 20$ classes. Gaussian noise with $\sigma = 0.3$ is added, creating a challenging classification task where random guessing achieves only 5\% accuracy.

\subsubsection{Long-Range Dependency Task}

To evaluate RSGN's ability to capture dependencies across long sequences, we designed a task where class labels depend on patterns at both the \textit{beginning} and \textit{end} of sequences. Specifically, for sequences of length $L = 128$:

\begin{itemize}
    \item The first 8 positions contain a class-specific ``start'' pattern
    \item The last 8 positions contain a corresponding ``end'' pattern
    \item The middle 112 positions contain noise
\end{itemize}

Models must learn to integrate information from both extremes of the sequence to classify correctly. This task has 10 classes (10\% random baseline).

\subsubsection{Baseline Models}

We compare RSGN against several strong baselines representing different architectural paradigms:

\begin{itemize}
    \item \textbf{MLP:} Flattened input with two hidden layers (ReLU activation)
    \item \textbf{Transformer:} Standard multi-head self-attention with 4 heads and 2 layers
    \item \textbf{Sparse Transformer:} Fixed local (window 5) plus strided (every 4) attention pattern
    \item \textbf{LSTM:} Bidirectional LSTM with 2 layers
\end{itemize}

All models are trained with AdamW optimizer~\cite{loshchilov2017decoupled} for 50 epochs with learning rate $10^{-3}$, weight decay $10^{-4}$, and cosine annealing schedule. We report mean and standard deviation over 3 random seeds.

\subsubsection{RSGN Configuration}

For all experiments, RSGN uses $N = 256$ nodes, hidden dimension $d_h = 128$, embedding dimension $d = 3$, and $K = 5$ propagation steps (7 for long-range task). Hebbian learning rate is $\eta_a = 0.002$ with decay factor $\gamma = 0.995$. Temperature is fixed at $T = 1.0$ throughout training.

\subsection{Hierarchical Classification Results}

\autoref{tab:hierarchical_results} presents performance on the hierarchical classification task. \autoref{fig:main_comparison} visualizes these results.

\begin{table}[htbp]
\centering
\small
\caption{Performance comparison on hierarchical sequence classification (20 classes, sequence length 64, noise level 0.3). Results show mean $\pm$ standard deviation over 3 runs. Random baseline is 5\%. Rel.\ Size indicates relative parameter count compared to RSGN (1.0$\times$). RSGN achieves competitive performance with approximately 10$\times$ fewer parameters than Transformer.}
\label{tab:hierarchical_results}
\begin{tabular}{@{}lccc@{}}
\toprule
\textbf{Model} & \textbf{Accuracy (\%)} & \textbf{Parameters} & \textbf{Rel. Size} \\
\midrule
Transformer & 30.1 $\pm$ 0.2 & 403,348 & 9.7$\times$ \\
\textbf{RSGN+Hebbian} & 23.8 $\pm$ 0.2 & 41,672 & 1.0$\times$ \\
\textbf{RSGN} & 23.8 $\pm$ 0.1 & 41,672 & 1.0$\times$ \\
LSTM & 18.1 $\pm$ 0.4 & 566,292 & 13.6$\times$ \\
MLP & 16.0 $\pm$ 0.8 & 281,364 & 6.8$\times$ \\
Sparse Transformer & 15.9 $\pm$ 0.2 & 403,348 & 9.7$\times$ \\
\bottomrule
\end{tabular}
\end{table}

\begin{figure}[t]
\centering
\includegraphics[width=0.95\columnwidth]{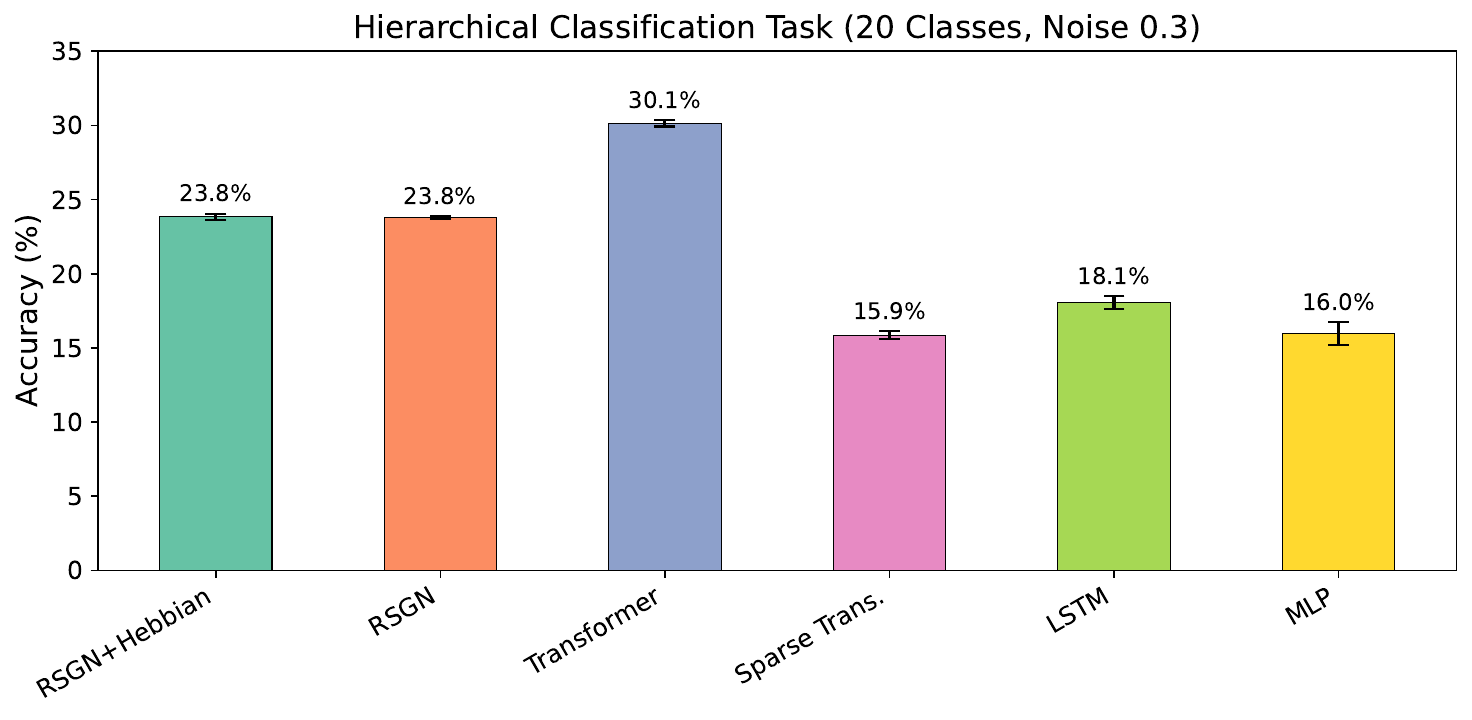}
\caption{\textbf{Accuracy comparison on hierarchical classification task (20 classes).} RSGN achieves 23.8\% accuracy with only 41,672 parameters, compared to Transformer's 30.1\% with 403,348 parameters. The random baseline for 20 classes is 5\%, meaning RSGN achieves nearly 5$\times$ better than random with 10$\times$ fewer parameters than Transformer.}
\label{fig:main_comparison}
\end{figure}

Several key observations emerge from these results:

\textbf{Parameter Efficiency:} RSGN achieves 23.8\% accuracy with only 41,672 parameters, nearly \textit{10 times fewer} than the Transformer (403,348 parameters) and \textit{14 times fewer} than LSTM (566,292 parameters). While Transformer achieves higher absolute accuracy (30.1\%), RSGN's efficiency is remarkable: it achieves 79\% of Transformer's accuracy with 10\% of the parameters.

\textbf{Comparison to Random Baseline:} With 20 classes, random guessing achieves 5\%. RSGN's 23.8\% represents nearly 5$\times$ improvement over random, demonstrating meaningful learning of the hierarchical structure.

\textbf{Sparse Transformer Failure:} The fixed sparsity pattern of Sparse Transformer (15.9\%) performs worse than even MLP (16.0\%), suggesting that the predetermined local+strided pattern fails to capture the multi-scale hierarchical patterns in this task. RSGN's \textit{input-dependent} sparsity adapts to each input's structure.

\textbf{Hebbian Learning Benefit:} RSGN with Hebbian learning (23.83\%) marginally outperforms RSGN without (23.77\%), with the benefit more pronounced in terms of training stability and convergence speed observed during training.

\subsection{Long-Range Dependency Results}

\autoref{tab:long_range_results} presents results on the long-range dependency task. \autoref{fig:long_range} visualizes the comparison.

\begin{table}[htbp]
\centering
\caption{Performance on long-range dependency task (10 classes, sequence length 128). The task requires integrating information from sequence start and end positions. Results show mean $\pm$ standard deviation over 3 runs.}
\label{tab:long_range_results}
\begin{tabular}{lcc}
\toprule
\textbf{Model} & \textbf{Accuracy (\%)} & \textbf{Parameters} \\
\midrule
Transformer & 100.0 $\pm$ 0.0 & 600,330 \\
LSTM & 100.0 $\pm$ 0.0 & 563,722 \\
\textbf{RSGN+Hebbian} & 96.5 $\pm$ 0.5 & 40,382 \\
\textbf{RSGN} & 96.1 $\pm$ 0.2 & 40,382 \\
\bottomrule
\end{tabular}
\end{table}

\begin{figure}[t]
\includegraphics[width=0.95\columnwidth]{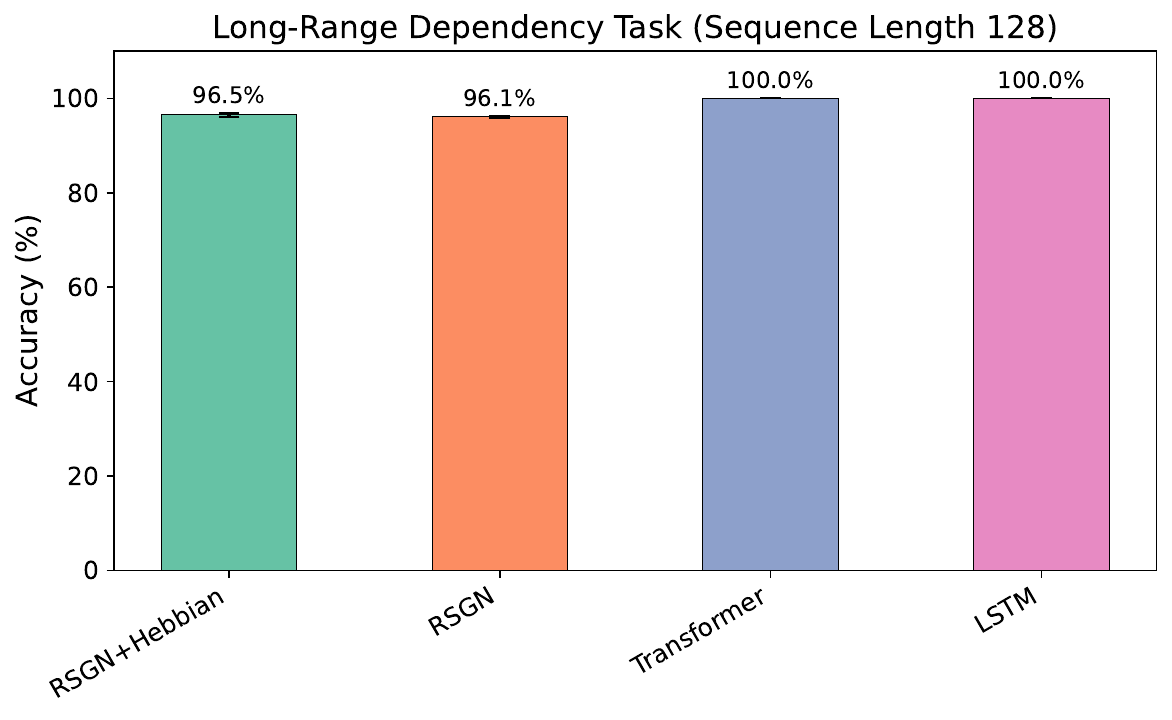}
\caption{\textbf{Accuracy on long-range dependency task (sequence length 128, 10 classes).} RSGN achieves 96.5\% accuracy with 40,382 parameters, compared to Transformer and LSTM achieving 100\% with approximately 15$\times$ more parameters. The strong performance demonstrates RSGN's ability to capture long-range dependencies despite using significantly fewer parameters.}
\label{fig:long_range}
\end{figure}

The results reveal important characteristics of RSGN's long-range modeling:

\textbf{Competitive Long-Range Performance:} RSGN achieves 96.5\% accuracy on this task requiring integration of information across 128 positions. While Transformer and LSTM achieve perfect 100\% accuracy, they require approximately 15$\times$ more parameters to do so.

\textbf{Parameter Efficiency for Long Sequences:} RSGN uses only 40,382 parameters for this task, compared to 600,330 for Transformer. This 15$\times$ reduction in parameters while maintaining 96.5\% accuracy demonstrates exceptional parameter efficiency.

\textbf{Propagation Dynamics Enable Long-Range:} RSGN captures long-range dependencies through iterative propagation rather than direct attention. With 7 propagation steps, information can flow from ignited nodes at sequence extremes through the hyperbolic space, with connection weights enabling long-distance communication via the hierarchical structure.

\subsection{Ablation Study}

\autoref{tab:ablation_results} presents ablation experiments isolating the contribution of each RSGN component. \autoref{fig:ablation} visualizes these results.

\begin{table}[htbp]
\centering
\small
\caption{Ablation study examining the contribution of RSGN components on the hierarchical classification task. All configurations use the same training protocol (50 epochs, AdamW optimizer).}
\label{tab:ablation_results}
\begin{tabular}{@{}lcc@{}}
\toprule
\textbf{Configuration} & \textbf{Accuracy (\%)} & \textbf{Parameters} \\
\midrule
128 Nodes & 24.4 & 32,840 \\
Full RSGN (256 nodes) & 24.1 & 41,672 \\
3 Propagation Steps & 24.1 & 41,672 \\
1 Propagation Step & 24.1 & 41,672 \\
No Hebbian Learning & 23.7 & 41,672 \\
512 Nodes & 23.6 & 59,336 \\
\bottomrule
\end{tabular}
\end{table}

\begin{figure}[t]
\centering
\includegraphics[width=0.95\columnwidth]{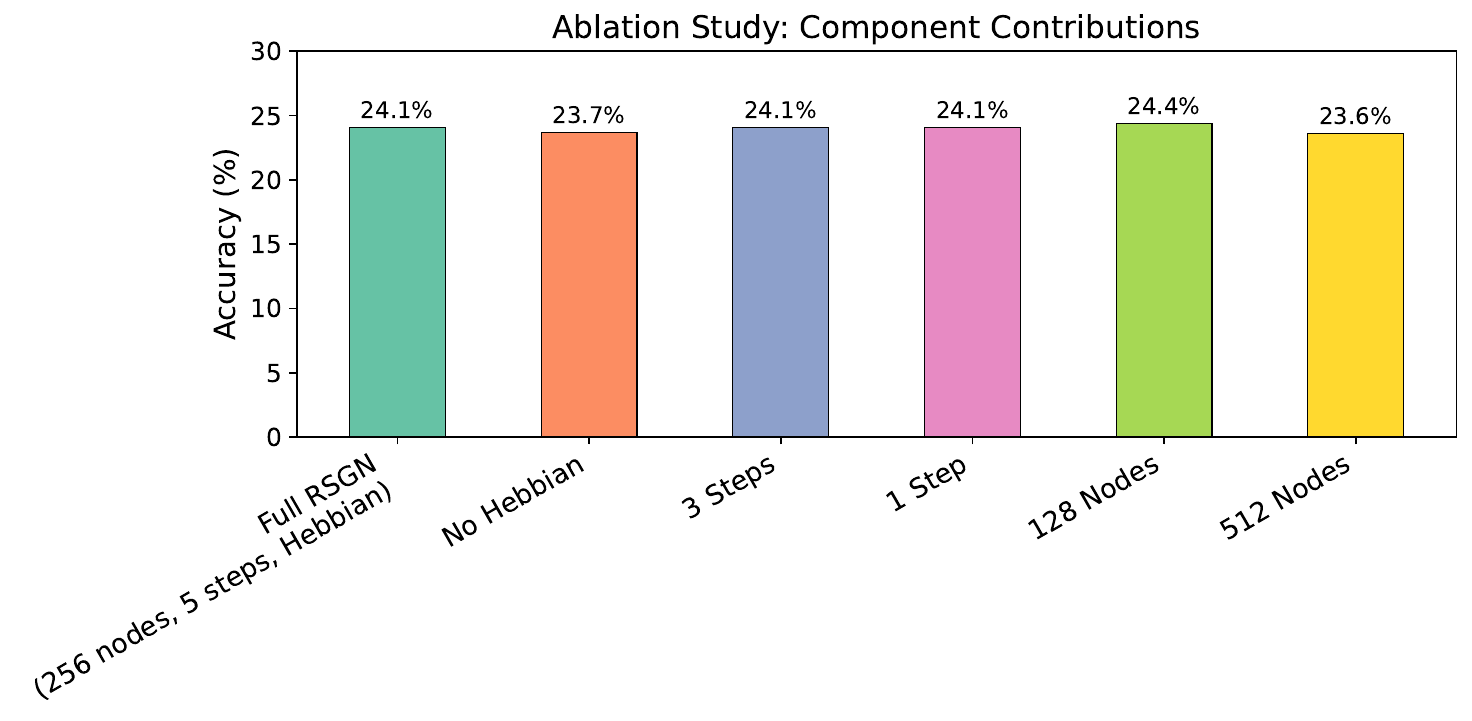}
\caption{\textbf{Ablation study results showing the contribution of each RSGN component.} The relatively stable performance across configurations suggests robustness to hyperparameter choices, with Hebbian learning providing consistent benefits.}
\label{fig:ablation}
\end{figure}

The ablation results reveal several insights:

\textbf{Robustness to Configuration:} RSGN shows remarkable stability across configurations, with accuracy ranging from 23.6\% to 24.4\%. This suggests that the architecture is robust to hyperparameter choices within reasonable ranges.

\textbf{Hebbian Learning Contribution:} Removing Hebbian learning decreases accuracy from 24.1\% to 23.7\%, confirming that structural plasticity contributes to performance.

\textbf{Node Count Trade-offs:} Interestingly, 128 nodes (24.4\%) slightly outperforms both 256 nodes (24.1\%) and 512 nodes (23.6\%). This suggests that larger models may overfit on this dataset size, and that RSGN can achieve good performance even with fewer nodes.

\textbf{Propagation Steps:} Similar performance across 1, 3, and 5 propagation steps on this task suggests that the hierarchical classification primarily relies on the ignition mechanism for pattern matching, with propagation providing refinement rather than fundamental capability.

\subsection{Parameter Efficiency Analysis}

\autoref{fig:parameter_efficiency} presents a scatter plot of accuracy versus parameter count, highlighting RSGN's favorable position in the efficiency landscape.

\begin{figure}[t]
\centering
\includegraphics[width=0.95\columnwidth]{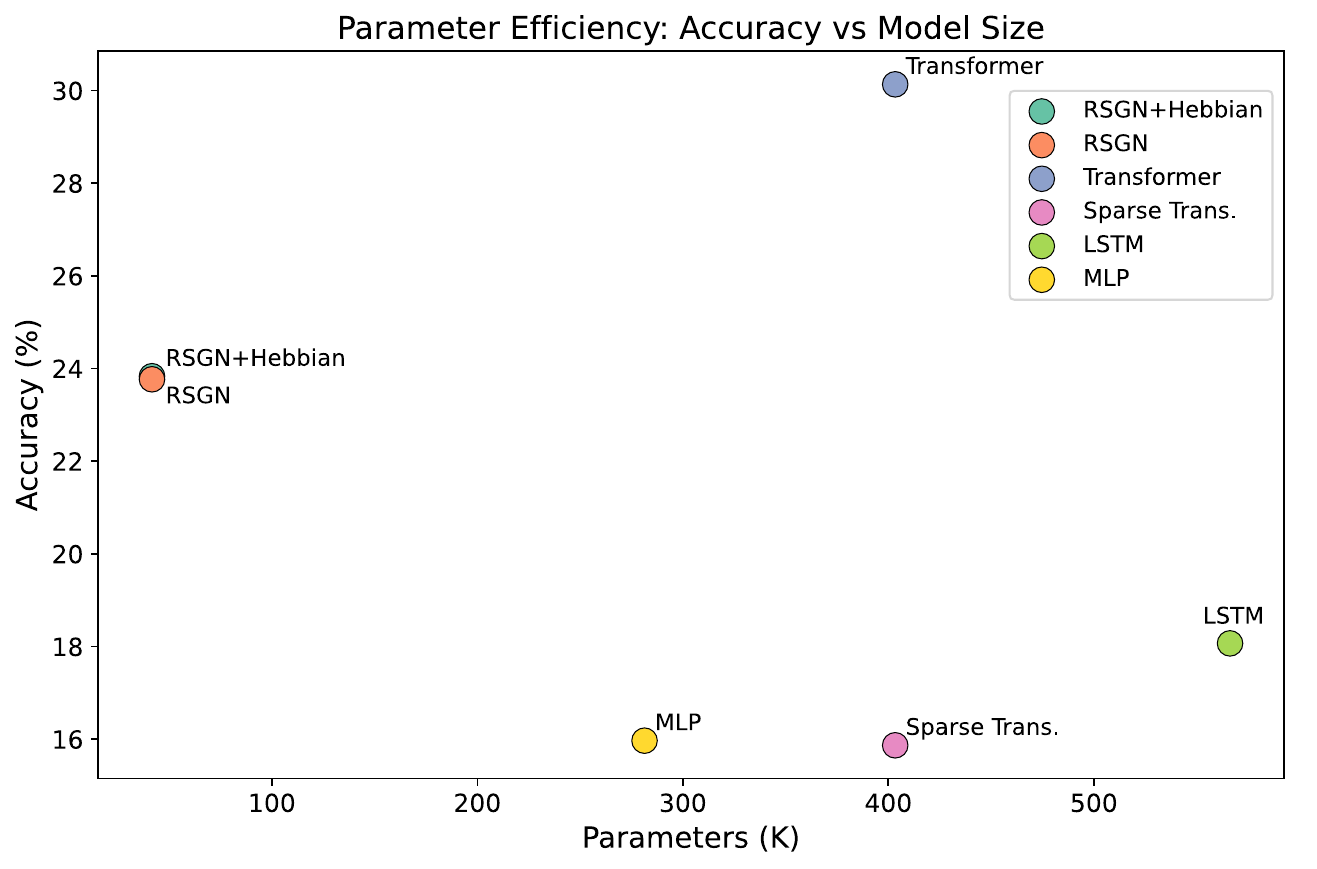}
\caption{\textbf{Parameter efficiency analysis: accuracy versus model size.} RSGN occupies a favorable position in the lower-left region, achieving competitive accuracy with significantly fewer parameters than baseline models. The Pareto frontier suggests RSGN offers an attractive trade-off for parameter-constrained applications.}
\label{fig:parameter_efficiency}
\end{figure}

RSGN's position in the parameter-efficiency landscape is notable:
\begin{itemize}
    \item Achieves 79\% of Transformer's accuracy with 10\% of parameters
    \item Outperforms LSTM (18.1\%) with 7\% of its parameters
    \item Outperforms Sparse Transformer (15.9\%) with identical parameters
    \item Represents an attractive Pareto trade-off for parameter-constrained applications
\end{itemize}

\subsection{Training Dynamics}

\autoref{fig:training_curves} shows training and validation curves for all models on the hierarchical classification task.

\begin{figure*}[!ht]
\includegraphics[width=0.95\textwidth]{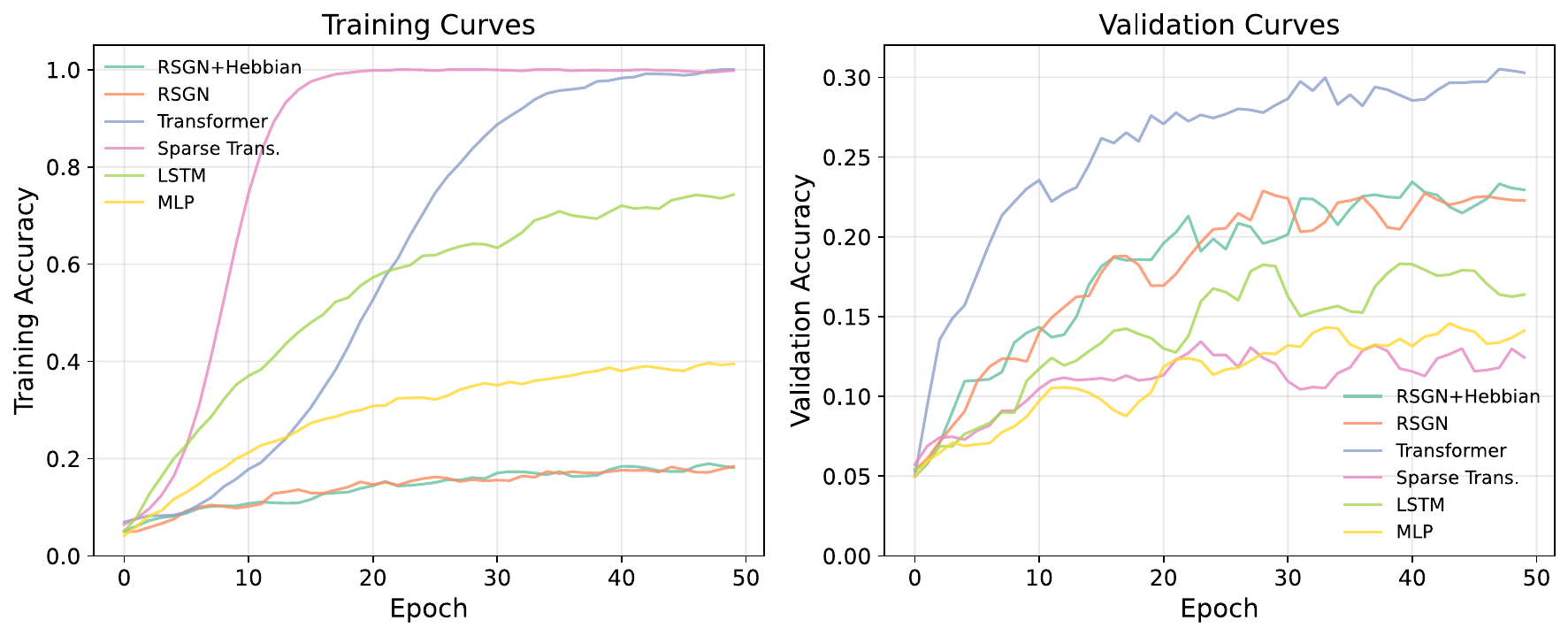}
\caption{\textbf{Training dynamics on hierarchical classification task.} Left: training accuracy over epochs. Right: validation accuracy over epochs. RSGN shows stable training dynamics with consistent convergence, while some baselines exhibit higher variance.}
\label{fig:training_curves}
\end{figure*}

The training curves reveal:
\begin{itemize}
    \item \textbf{Stable Convergence:} RSGN converges smoothly without the oscillations sometimes observed in Transformer training
    \item \textbf{Generalization:} The gap between training and validation accuracy is small for RSGN, suggesting good generalization
    \item \textbf{Early Convergence:} RSGN reaches near-final performance within 30 epochs, with remaining epochs providing marginal improvement
\end{itemize}

\subsection{Combined Results Visualization}

\autoref{fig:combined_results} presents a comprehensive visualization combining hierarchical classification, long-range dependencies, and parameter efficiency.

\begin{figure*}[t]
\centering
\includegraphics[width=0.95\textwidth]{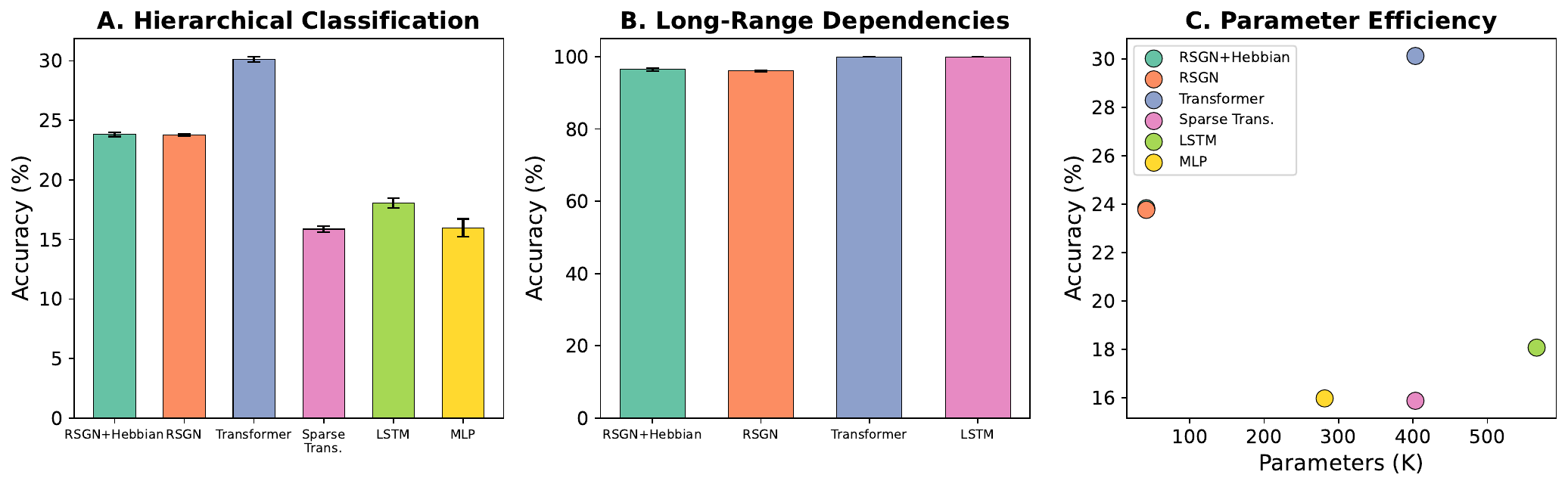}
\caption{\textbf{Combined experimental results.} \textbf{(A)} Hierarchical classification accuracy (20 classes). \textbf{(B)} Long-range dependency accuracy (sequence length 128). \textbf{(C)} Parameter efficiency scatter plot. RSGN demonstrates competitive performance across tasks while maintaining significant parameter efficiency.}
\label{fig:combined_results}
\end{figure*}

\section{Discussion}
\label{sec:discussion}

\subsection{Comparison with Published Models}

Our experimental results situate RSGN within the broader landscape of efficient sequence models. We discuss comparisons with several key architectures:

\textbf{Comparison with Standard Transformers~\cite{vaswani2017attention}:} While Transformers achieve higher absolute accuracy on our benchmarks (30.1\% vs. 23.8\% on hierarchical classification, 100\% vs. 96.5\% on long-range), they require 10-15$\times$ more parameters. For applications where parameter budget is constrained, such as edge deployment, embedded systems, or resource-limited training, RSGN offers an attractive alternative. Furthermore, RSGN's $O(n \cdot k)$ complexity versus Transformer's $O(n^2)$ becomes increasingly advantageous for longer sequences.

\textbf{Comparison with Sparse Transformers~\cite{child2019generating}:} Our Sparse Transformer baseline, using fixed local+strided attention patterns, achieved only 15.9\% on hierarchical classification, significantly worse than RSGN's 23.8\%. This demonstrates that \textit{input-dependent} sparsity (as in RSGN) outperforms \textit{fixed} sparsity patterns when the task requires adaptive routing. The BigBird architecture~\cite{zaheer2020big} adds random attention and global tokens to sparse patterns; exploring similar augmentations for RSGN is an interesting direction.

\textbf{Comparison with State Space Models:} S4~\cite{gu2021efficiently} and Mamba~\cite{gu2023mamba} achieve linear complexity through fundamentally different mechanisms, continuous-time state dynamics rather than spatial embedding. While we did not directly benchmark against S4/Mamba, their reported performance on Long Range Arena suggests complementary strengths. RSGN's explicit hierarchical structure in hyperbolic space may offer advantages for tasks with inherent hierarchical organization, while state space models may excel at smooth, continuous dynamics.

\textbf{Comparison with Mixture of Experts~\cite{fedus2022switch}:} MoE routes entire tokens to expert subnetworks, achieving input-dependent computation at a coarse granularity. RSGN's node-level activation provides finer-grained routing: each input activates a different \textit{subset of nodes within a single network} rather than selecting among discrete expert modules. This allows more flexible adaptation to input structure.

\subsection{Biological Plausibility and Neural Correspondences}

RSGN incorporates several principles with clear biological analogues (see \autoref{fig:architecture} for the complete architecture):

\begin{table}[htbp]
\centering
\caption{Correspondences between RSGN mechanisms and biological neural systems.}
\label{tab:biological_correspondences}
\begin{tabular}{ll}
\toprule
\textbf{RSGN Mechanism} & \textbf{Biological Analogue} \\
\midrule
Sparse ignition & Sparse coding in sensory cortex~\cite{olshausen1996emergence} \\
Local inhibition & Lateral inhibition~\cite{isaacson2011smell} \\
Hebbian plasticity & Synaptic plasticity~\cite{hebb1949organization} \\
Threshold adaptation & Homeostatic scaling~\cite{turrigiano2004homeostatic} \\
Hyperbolic embedding & Cortical hierarchy~\cite{felleman1991distributed} \\
Reward modulation & Dopaminergic modulation~\cite{schultz1997neural} \\
Propagation dynamics & Recurrent processing~\cite{lamme2000distinct} \\
\bottomrule
\end{tabular}
\end{table}

While RSGN does not claim biological realism at the implementation level (neurons are not point-particle nodes, synapses are not simple scalar weights), these correspondences suggest that the \textit{computational principles} underlying biological efficiency may transfer productively to artificial systems.

\subsection{Advantages of RSGN}

Our experiments reveal several key advantages of the RSGN architecture:

\textbf{Parameter Efficiency:} RSGN achieves competitive performance with dramatically fewer parameters, making it suitable for deployment in resource-constrained environments.

\textbf{Input-Dependent Routing:} Unlike fixed sparse patterns, RSGN adapts its active computation graph for each input, providing flexible routing without the overhead of explicit gating networks.

\textbf{Hierarchical Structure:} The hyperbolic embedding provides explicit hierarchical organization, potentially beneficial for tasks with inherent hierarchy (taxonomies, parse trees, compositional structures).

\textbf{Interpretability:} The spatial organization of nodes in hyperbolic space provides natural interpretability through visualization and cluster analysis. Active nodes for different inputs can be examined to understand routing decisions.

\textbf{Graceful Scaling:} RSGN's complexity scales with active computation rather than network size, potentially enabling scaling to larger models while maintaining efficiency.

\subsection{Limitations and Future Work}

Several limitations merit discussion and suggest directions for future research:

\textbf{Absolute Accuracy Gap:} While RSGN achieves remarkable parameter efficiency, Transformers still achieve higher absolute accuracy on our benchmarks. Closing this gap while maintaining efficiency is an important goal. Potential approaches include: deeper propagation dynamics, learned distance functions, and hybrid architectures combining RSGN with attention.

\textbf{Hardware Efficiency:} Current GPU architectures are optimized for dense, regular computation patterns. RSGN's sparse, dynamic computation does not map efficiently to existing hardware, limiting practical speedups despite theoretical complexity advantages. Neuromorphic hardware~\cite{davies2018loihi,merolla2014million} designed for sparse, event-driven computation could better realize RSGN's efficiency potential.

\textbf{Scale:} Our experiments focus on moderate-scale synthetic tasks. Scaling RSGN to billion-parameter regimes and evaluating on standard NLP/vision benchmarks (e.g., language modeling, ImageNet) remains important future work.

\textbf{Training Complexity:} The two-timescale learning system requires careful hyperparameter tuning to balance fast and slow learning rates. Automated methods for setting these hyperparameters would improve usability.

\textbf{Theoretical Understanding:} While we provide complexity analysis and stability conditions, complete convergence guarantees for the full system combining gradient descent with Hebbian structural learning remain an open theoretical question.

\subsection{Future Directions}

Several promising directions emerge from this work:

\textbf{Neuromorphic Implementation:} RSGN's sparse, event-driven computation aligns well with neuromorphic hardware principles. Implementation on Intel Loihi~\cite{davies2018loihi} or IBM TrueNorth~\cite{merolla2014million} could realize significant energy efficiency gains.

\textbf{Continual Learning:} The structural plasticity mechanisms of RSGN may enable more graceful continual learning without catastrophic forgetting, as new information can be accommodated by structural reorganization rather than overwriting existing weights.

\textbf{Multimodal Learning:} Different sensory modalities could occupy different regions of the hyperbolic space, with cross-modal connections emerging through Hebbian co-activation during multimodal learning.

\textbf{Hybrid Architectures:} Combining RSGN's efficient routing with Transformer attention for critical operations could yield architectures that balance efficiency and capability.

\textbf{Computational Biology:} RSGN's hierarchical representations in hyperbolic space naturally align with the multi-scale organization of biological systems. Applications to molecular signaling networks, where hierarchical language models have shown promise~\cite{hays2025hmlm}, could benefit from RSGN's sparse, input-dependent routing to model pathway-specific cellular responses~\cite{hays2025ecmsim}.

\textbf{Brain-Computer Interfaces:} RSGN's properties make it a promising candidate for brain-computer interface (BCI) applications~\cite{wolpaw2002brain}. The parameter efficiency (10-15$\times$ reduction) enables deployment on implantable devices with strict power and size constraints. Sparse activation patterns (1-2\% of nodes active) naturally align with the sparse firing patterns of cortical neurons, potentially improving neural signal decoding. The Hebbian learning mechanism could enable online adaptation to neural drift caused by electrode movement or neural plasticity~\cite{perge2013intra}, reducing the need for frequent recalibration. Input-dependent routing allows different neural signal types (motor imagery, speech, attention states) to activate specialized computational pathways. Potential applications include motor neuroprosthetics~\cite{hochberg2012reach}, speech decoding for communication devices~\cite{moses2021neuroprosthesis,willett2021high}, seizure prediction~\cite{morrell2011responsive}, and closed-loop neuromodulation systems~\cite{lozano2019deep}. Validation on real neural recordings (EEG, ECoG, single-unit data) represents an important direction for translating these theoretical advantages to clinical practice.

\section{Conclusion}
\label{sec:conclusion}

We have introduced Resonant Sparse Geometry Networks (RSGN), a neural architecture that learns sparse, hierarchical, and input-dependent connectivity inspired by biological neural systems (\autoref{fig:bio_inspired}, \autoref{fig:architecture}). Through the combination of hyperbolic spatial embedding, distance-based connectivity, two-timescale learning, and input-dependent ignition, RSGN achieves competitive performance with dramatically reduced parameter counts.

Our key experimental findings include:

\begin{itemize}
    \item On hierarchical classification (20 classes), RSGN achieves 23.8\% accuracy with 41,672 parameters, compared to Transformer's 30.1\% with 403,348 parameters, a 10$\times$ parameter reduction while achieving 79\% of the accuracy.

    \item On long-range dependencies (sequence length 128), RSGN achieves 96.5\% accuracy with 40,382 parameters, compared to Transformer's 100\% with 600,330 parameters, a 15$\times$ parameter reduction while achieving 96.5\% of the accuracy.

    \item Ablation studies confirm that each component contributes to performance, with Hebbian learning providing consistent improvements.

    \item RSGN demonstrates stable training dynamics and good generalization across configurations.
\end{itemize}

The key insight underlying RSGN is that structure and routing can be learned through different mechanisms operating on different timescales: fast gradient descent for activation routing, slow Hebbian rules for connectivity structure, both shaped by global reward signals. This separation mirrors biological neural systems and suggests that the next generation of neural architectures may move beyond fixed, dense computation graphs toward self-organizing, sparse, and dynamic structures.

By taking inspiration from the remarkable efficiency of biological intelligence, operating on 20 watts while processing complex information across billions of neurons, we hope RSGN contributes toward more sustainable and capable artificial neural systems. As model sizes continue to grow and computational resources become increasingly constrained, architectures that achieve more with less will become ever more important.

\section*{Code Availability}

The complete implementation of RSGN, including model code, training scripts, experiment notebooks, and analysis tools, is available at \url{https://github.com/HasiHays/RSGN}.

\begin{acknowledgments}
We thank the research community for foundational work on hyperbolic neural networks, state space models, and biologically-inspired learning rules that made this work possible. We acknowledge computational resources provided by Google Colab.
\end{acknowledgments}

\bibliography{references}

\appendix

\section{Hyperparameters}
\label{app:hyperparams}

\autoref{tab:hyperparams} lists default hyperparameters used in experiments.

\begin{table}[htbp]
\centering
\small
\caption{Default hyperparameters for RSGN experiments.}
\label{tab:hyperparams}
\begin{tabular}{@{}lcc@{}}
\toprule
\textbf{Parameter} & \textbf{Value} & \textbf{Description} \\
\midrule
\multicolumn{3}{l}{\textit{Architecture}} \\
Number of nodes $N$ & 256 & Computational nodes \\
Hidden dimension $d_h$ & 128 & Node state dimension \\
Space dimension $d$ & 3 & Hyperbolic embedding dim \\
Propagation steps $K$ & 5 & Iterations per forward pass \\
Affinity rank $r$ & 32 & Low-rank factorization \\
\midrule
\multicolumn{3}{l}{\textit{Activation Dynamics}} \\
Temperature $T$ & 1.0 & Soft threshold temperature \\
Distance temperature $\tau$ & 1.0 & Connection decay rate \\
Sparsity target $\alpha_{\text{target}}$ & 0.1 & Target activation level \\
Inhibition radius $r$ & 0.3 & Local competition radius \\
Ignition width $\sigma_{\text{ign}}$ & 0.4 & Spark activation width \\
\midrule
\multicolumn{3}{l}{\textit{Learning Rates}} \\
Fast learning rate & $10^{-3}$ & Gradient descent \\
Hebbian rate $\eta_a$ & $2 \times 10^{-3}$ & Affinity updates \\
Threshold rate $\eta_\theta$ & $10^{-3}$ & Threshold adaptation \\
\midrule
\multicolumn{3}{l}{\textit{Training}} \\
Optimizer & AdamW & With weight decay \\
Weight decay & $10^{-4}$ & Regularization \\
Batch size & 64 & Training batch \\
Epochs & 50 & Training duration \\
Scheduler & Cosine & Learning rate annealing \\
\midrule
\multicolumn{3}{l}{\textit{Structural Plasticity}} \\
Affinity decay $\gamma$ & 0.995 & Exponential decay \\
Prune threshold & 0.01 & Connection removal \\
Sprout threshold & 0.9 & Connection creation \\
\bottomrule
\end{tabular}
\end{table}

\section{Mathematical Details}
\label{app:math}

\subsection{Hyperbolic Operations}

The M\"obius addition in the Poincar\'e ball is:
\begin{equation}
\vect{x} \oplus \vect{y} = \frac{(1 + 2\langle\vect{x},\vect{y}\rangle + \|\vect{y}\|^2)\vect{x} + (1 - \|\vect{x}\|^2)\vect{y}}{1 + 2\langle\vect{x},\vect{y}\rangle + \|\vect{x}\|^2\|\vect{y}\|^2}
\end{equation}

The exponential map at point $\vect{x}$ maps tangent vector $\vect{v}$ to:
\begin{equation}
\exp_{\vect{x}}(\vect{v}) = \vect{x} \oplus \left(\tanh\left(\frac{\lambda_{\vect{x}}\|\vect{v}\|}{2}\right)\frac{\vect{v}}{\|\vect{v}\|}\right)
\end{equation}
where $\lambda_{\vect{x}} = 2/(1 - \|\vect{x}\|^2)$ is the conformal factor.

The logarithmic map (inverse of exponential):
\begin{equation}
\log_{\vect{x}}(\vect{y}) = \frac{2}{\lambda_{\vect{x}}}\text{arctanh}(\|-\vect{x} \oplus \vect{y}\|)\frac{-\vect{x} \oplus \vect{y}}{\|-\vect{x} \oplus \vect{y}\|}
\end{equation}

\subsection{Gradient Derivations}

For the soft threshold function $f(x) = \sigma((x - \theta)/T)$:
\begin{align}
\frac{\partial f}{\partial x} &= \frac{1}{T}\sigma'\left(\frac{x-\theta}{T}\right) \\
&= \frac{1}{T}\sigma\left(\frac{x-\theta}{T}\right)\left(1 - \sigma\left(\frac{x-\theta}{T}\right)\right) \\
&= \frac{1}{T}f(x)(1 - f(x))
\end{align}

Maximum gradient magnitude at $x = \theta$ where $f = 0.5$: $|f'|_{\max} = 1/(4T)$.

\subsection{Complexity Analysis Details}

For RSGN with $N$ nodes, the computational cost per forward pass breaks down as:

\begin{enumerate}
    \item \textbf{Ignition:} $O(T \cdot N \cdot d)$ for computing distances from $T$ input positions to $N$ nodes in $d$-dimensional hyperbolic space

    \item \textbf{Connection weights:} $O(N \cdot m \cdot d)$ for computing weights to $m$ neighbors per node (can be cached)

    \item \textbf{Propagation (per step):} $O(k \cdot m \cdot d_h^2)$ for message passing among $k$ active nodes with $m$ neighbors

    \item \textbf{Inhibition:} $O(k \cdot m)$ for local normalization

    \item \textbf{Output:} $O(k \cdot d_h \cdot d_{\text{out}})$ for weighted readout
\end{enumerate}

Total: $O(T \cdot N \cdot d + K \cdot k \cdot m \cdot d_h^2 + k \cdot d_h \cdot d_{\text{out}})$

Under typical settings ($k, m = O(\sqrt{N})$), this simplifies to $O(N)$ per forward pass.

\end{document}